\documentclass[english,a4paper,12pt]{article}
\RequirePackage{amsmath}

\usepackage{amsxtra, amsfonts, amssymb, amstext, mathtools}
\usepackage{amsthm}
\usepackage{booktabs}
\usepackage{nicefrac}
\usepackage{xspace}
\usepackage{url}
\usepackage{graphics,color}
\usepackage[algo2e,ruled,vlined,linesnumbered]{algorithm2e}
\usepackage{wrapfig}

\newtheorem{theorem}{Theorem}
\newtheorem{lemma}[theorem]{Lemma}

\newtheorem{corollary}[theorem]{Corollary}

\newcommand{\oea}{\mbox{$(1 + 1)$~EA}\xspace}

\newcommand{\mclea}{\mbox{${(\mu,\lambda)}$~EA}\xspace}
\newcommand{\oclea}{\mbox{$(1,\lambda)$~EA}\xspace}

\newcommand{\onemax}{\textsc{OneMax}\xspace}

\DeclareMathOperator{\fp}{fp}
\DeclareMathOperator{\mut}{mut}
\DeclareMathOperator{\sel}{sel}

\DeclareMathOperator{\poly}{poly}

\newcommand{\R}{\ensuremath{\mathbb{R}}}

\newcommand{\N}{\ensuremath{\mathbb{N}}} 
\newcommand{\Z}{\ensuremath{\mathbb{Z}}}

\newcommand{\calP}{\ensuremath{\mathcal{P}}}

\DeclareMathOperator{\Bin}{Bin}

\newcommand{\eps}{\varepsilon}

\newcommand{\assign}{\leftarrow}

\begin{document}
\title{\vspace*{-1cm}Lower Bounds for Non-Elitist Evolutionary Algorithms via Negative Multiplicative Drift\thanks{Extended version of a paper that appeared in the proceedings of PPSN 2020~\cite{Doerr20ppsnLB}. This version contains all proofs (for reasons of space, in~\cite{Doerr20ppsnLB} only Lemma~1 and Theorem~2 were proven), a new section on standard bit mutation with random mutation rates, and several additional details.}}

\author{Benjamin Doerr\setcounter{footnote}{6}\thanks{Laboratoire d'Informatique (LIX), CNRS, \'Ecole Polytechnique, Institut Polytechnique de Paris, Palaiseau, France
}}

\date{}

\maketitle

{\sloppy

\begin{abstract}
  A decent number of lower bounds for non-elitist population-based evolutionary algorithms has been shown by now. Most of them are technically demanding due to the (hard to avoid) use of negative drift theorems -- general results which translate an expected movement away from the target into a high hitting time. 
  
  We propose a simple negative drift theorem for multiplicative drift scenarios and show that it can simplify existing analyses. We discuss in more detail Lehre's (PPSN 2010) \emph{negative drift in populations} method, one of the most general tools to prove lower bounds on the runtime of non-elitist mutation-based evolutionary algorithms for discrete search spaces. Together with other arguments, we obtain an alternative and simpler proof of this result, which also strengthens and simplifies this method. In particular, now only three of the five technical conditions of the previous result have to be verified. The lower bounds we obtain are explicit instead of only asymptotic. This allows to compute concrete lower bounds for concrete algorithms, but also enables us to show that super-polynomial runtimes appear already when the reproduction rate is only a $(1 - \omega(n^{-1/2}))$ factor below the threshold. For the special case of algorithms using standard bit mutation with a random mutation rate (called uniform mixing in the language of hyper-heuristics), we prove the result stated by Dang and Lehre (PPSN 2016) and extend it to mutation rates other than $\Theta(1/n)$, which includes the heavy-tailed mutation operator proposed by Doerr, Le, Makhmara, and Nguyen (GECCO 2017). We finally use our method and a novel domination argument to show an exponential lower bound for the runtime of the mutation-only simple genetic algorithm on \onemax for arbitrary population size.
%
\end{abstract}

\section{Introduction}

Lower bounds for the runtimes of evolutionary algorithms are important as they can warn the algorithm user that certain algorithms or certain parameter settings will not lead to good solutions in acceptable time. Unfortunately, the existing methods to obtain such results, for non-elitist algorithms in particular, are very technical and thus difficult to use.

One reason for this high complexity is the use of drift analysis, which seems hard to circumvent. Drift analysis~\cite{Lengler20bookchapter} is a set of tools that all try to derive useful information on a hitting time (e.g., the first time a solution of a certain quality is found) from information on the expected progress in one iteration. The hope is that the progress in a single iteration can be analyzed with only moderate difficulty and then the drift theorem does the remaining work. While more direct analysis methods exist and have been successfully used for simple algorithms, for population-based algorithms and in particular non-elitist ones, it is hard to imagine that the complicated population dynamics can be captured in proofs not using more advanced tools such as drift analysis. 

Drift analysis has been used with great success to prove upper bounds on runtimes of evolutionary algorithms. Tools such as the additive~\cite{HeY01}, multiplicative~\cite{DoerrJW12algo}, and variable drift theorem~\cite{MitavskiyRC09,Johannsen10} all allow to easily obtain an upper bound on a hitting time solely from the expected progress in one iteration. Unfortunately, proving matching lower bounds is much harder since here the drift theorems require additional technical assumptions on the distribution of the progress in one iteration. This is even more true in the case of so-called \emph{negative drift}, where the drift is away from the target and we aim at proving a high lower bound on the hitting time. 

In this work, we propose a very simple negative drift theorem for the case of multiplicative drift (Lemma~\ref{lem:nddrift}). We brief{}ly show that this result can ease two classic lower bound analyses (also in Section~\ref{sec:nddrift}).

In more detail, we use the new drift theorem (and some more arguments) to rework Lehre's \emph{negative drift in populations} method~\cite{Lehre10}. This highly general analysis method allows to show exponential lower bounds on the runtime of a large class of evolutionary algorithms solely by comparing the so-called reproduction rate of individuals in the population with a threshold that depends only on the mutation rate. 

The downside of Lehre's method is that both the result and its proof are very technical. To apply the general result (and not the specialization to algorithms using standard bit mutation), five technical conditions need to be verified, which requires the user to choose suitable values for six different constants; these have an influence on the lower bound one obtains. This renders the method of Lehre hard to use. Among the 54 citations to~\cite{Lehre10} (according to Google scholar on June 9, 2020), only the two works~\cite{Lehre11,DangL16ppsn} apply this method. To hopefully ease future analyses of negative drift in populations, we revisit this method and obtain the following improvements.

\emph{A simpler result:} We manage to show essentially the same lower bound by only verifying three of the five conditions of Lehre's result (Theorems~\ref{thm:main} and~\ref{thm:main2}). This also reduces the number of constants one needs to choose from six to four.

\emph{A non-asymptotic result:} Our result gives explicit lower bounds, that is, free from asymptotic notation or unspecified constants. 
Consequently, our specialization to algorithms using standard bit mutation (Theorem~\ref{thm:sbm}) also gives explicit bounds. This allows one to prove concrete bounds for specific situations (e.g., that the \mclea with $\lambda = 2\mu$ needs more than 13 million fitness evaluations to find the optimum of the \onemax problem defined over bit strings of length $n=500$, see the example following Theorem~\ref{thm:sbm}) and gives more fine-grained theoretical results (by choosing Lehre's constant $\delta$ as a suitable function of the problems size, we show that a super-polynomial runtime behavior is observed already when the reproduction rate is only a $(1 - \omega(n^{1/2}))$ factor below the threshold, see Corollary~\ref{cor:sbm}). With the absence of asymptotic notation, we can also analyze algorithms using standard bit mutation with a mutation rate chosen randomly from a discrete set of alternatives (Section~\ref{sec:randsbm}). Such a result was stated by Dang and Lehre~\cite{DangL16ppsn}, however only for mutation rates that are $\Theta(1/n)$. Our result does not need this restriction and thus, for example, applies also to the heavy-tailed mutation operator proposed in~\cite{DoerrLMN17}.

\emph{A simple proof:} Besides the important aspect that a proof guarantees the result to be mathematically correct, an understandable proof can also tell us why a result is correct and give further insights into working principles of algorithms. While every reader will have a different view on how the ideal proof looks like, we felt that Lehre's proof, combining several deep and abstract tools such as multi-type branching processes, eigenvalue arguments, and Hajek's drift theorem~\cite{Hajek82}, does not easily give a broader understanding of the proof mechanics and the working principles of the algorithms analyzed. Our proof, based on a simple potential function argument together with our negative drift theorem, hopefully is more accessible.

Finally, we analyze an algorithm using fitness proportionate selection. The negative drift in populations method is not immediately applicable to such algorithms since it is hard to provide a general unconditional upper bound on the reproduction rate: If all but one individual have a very low fitness, then this best individual has a high reproduction rate. We therefore show that at all times all search points are at least as good (in a stochastic domination sense) as random search points. This allows to argue that the reproduction rates are low and then gives a simple proof of an exponential lower bound for the mutation-only simple genetic algorithm (simple GA) with arbitrary population size optimizing the simple \onemax benchmark, improving over the mildly sub-exponential lower bound in~\cite{NeumannOW09} and the exponential lower bound  for large population sizes only in~\cite{Lehre11}.

\subsection*{Related Works}

A number of different drift theorems dealing with negative drift have been proven so far, among others, in~\cite{HappJKN08,OlivetoW11,OlivetoW12,RoweS14,OlivetoW15,Kotzing16,LenglerS18,Witt19} (note that in early works, the name ``simplified drift theorem'' was used for such results). They all require some additional assumptions on the distribution of the one-step progress, which makes them non-trivial to use. We refer to~\cite[Section~2.4.3]{Lengler20bookchapter} for more details. Another approach to negative drift was used in~\cite{AntipovDY19,Doerr19foga,Doerr20gecco}. There the original process was transformed suitably (via an exponential function), but in a way that the drift of the new process still is negative or at most a small constant. To this transformed process the lower bound version of the additive drift theorem~\cite{HeY01} was applied, which gave large lower bounds since the target, due to the exponential rescaling, now was far from the starting point of the process.

In terms of lower bounds for non-elitist algorithms, besides Lehre's general result~\cite{Lehre10}, the following results for particular algorithms exist (always, $n$ is the problem size, $\eps$ can be any positive constant, and $e \approx 2.718$ is the base of the natural logarithm). J\"agersk\"upper and Storch~\cite[Theorem~1]{JagerskupperS07} showed that the \oclea with $\lambda \le \frac{1}{14} \ln(n)$ is inefficient on any pseudo-Boolean function with a unique optimum. The asymptotically tight condition $\lambda \le (1-\eps) \log_{\frac{e}{e-1}} n$ to yield a super-polynomial runtime was given by Rowe and Sudholt~\cite{RoweS14}. Happ, Johannsen, Klein, and Neumann~\cite{HappJKN08} showed that two simple (1+1)-type hillclimbers with fitness proportionate selection cannot optimize efficiently any linear function with positive weights. Neumann, Oliveto, and Witt~\cite{NeumannOW09} showed that a mutation-only variant of the simple GA with fitness proportionate selection is inefficient on the \onemax function when the population size $\mu$ is at most polynomial, and it is inefficient on any pseudo-Boolean function with unique global optimum when $\mu \le \frac 14 \ln(n)$. The mildly subexponential lower bound for \onemax was improved to an exponential lower bound by Lehre~\cite{Lehre11}, but only for $\mu \ge n^3$. In a series of remarkable works up to~\cite{OlivetoW15}, Oliveto and Witt showed that the true simple GA using crossover cannot optimize \onemax efficiently when $\mu \le n^{\frac 14 - \eps}$. None of these results gives an explicit lower bound or specifies the base of the exponential function. In~\cite{AntipovDY19}, an explicit lower bound for the runtime of the \mclea is proven (but stated only in the proof of Theorem~3.1 in~\cite{AntipovDY19}). Section~3 of~\cite{AntipovDY19} bears some similarity with ours, in fact, one can argue that our work extends~\cite[Section~3]{AntipovDY19} from a particular algorithm to the general class of population-based processes regarded by Lehre~\cite{Lehre10} (where, naturally, \cite{AntipovDY19} did not have the negative multiplicative drift result and therefore did not obtain bounds that hold with high probability).


\section{Notation and Preliminaries}\label{sec:prelims}

In terms of basic notation, we write $[a..b] := \{z \in \Z \mid a \le z \le b\}$. We recall the definition of the \onemax benchmark function
\[\onemax : \{0,1\}^n \to \R; x = (x_1, \dots, x_n) \mapsto \sum_{i=1}^n x_i,\]
which counts the number of ones in the argument. We denote the \emph{Hamming distance} of two bit strings $x, y \in \{0,1\}^n$ by 
\[H(x,y) = |\{i \in [1..n] \mid x_i \neq y_i\}|.\]

The classic mutation operator \emph{standard bit mutation} creates an offspring by flipping each bit of the parent independently with some probability $p$, which is called \emph{mutation rate}.

We shall twice need the notion of stochastic domination and its relation to standard bit mutation, so we quickly collect these ingredients of our proofs. We refer to~\cite{Doerr19tcs} for more details on stochastic domination and its use in runtime analysis. 

For two real-valued random variables $X$ and $Y$, we say that $Y$ \emph{stochastically dominates} $X$, written as $X \preceq Y$, if for all $\lambda \in \R$ we have $\Pr[Y \le \lambda] \le \Pr[X \le \lambda]$. Stochastic domination is a very flexible way of saying that $Y$ is larger than $X$. It implies $E[X] \le E[Y]$, but not only this, we also have $E[f(X)] \le E[f(Y)]$ for any monotonically increasing function $f$. 

\begin{lemma}\label{lem:mondom}
  Let $X, Y$ be two random variables taking values in some set $\Omega \subseteq \R$. Let $f : \Omega \to \R$ be monotonically increasing, that is, we have $f(x) \le f(y)$ for all $x,y \in \Omega$ with $x \le y$. Then $E[f(X)] \le E[f(Y)]$. 
\end{lemma}

Significantly improving over previous related arguments in~\cite[Section~5]{DrosteJW00iecon} and~\cite[Lemma~13]{DoerrJW12algo}, Witt showed the following natural domination argument \cite[Lemma~6.1]{Witt13} for offspring generated via standard bit mutation with mutation rate at most $\frac 12$. We note that the result is formulated in~\cite{Witt13} only for $x^* = (1, \dots, 1)$, but the proof in~\cite{Witt13} or a symmetry argument immediately shows the following general version. 

\begin{lemma}\label{lem:witt}
  Let $x^*, x, y \in \{0,1\}^n$ with $H(x,x^*) \ge H(y,x^*)$. Let $x'$ and $y'$ be random search points obtained from $x$ and $y$ via standard bit mutation with mutation rate $p \le \frac 12$. Then 
  \[H(x'x^*) \succeq H(y',x^*).\]
\end{lemma}

\section{Negative Multiplicative Drift}\label{sec:nddrift}

The following elementary result allows to prove lower bounds on the time to reach a target in the presence of multiplicative drift away from the target. While looking innocent, it has the potential to replace the more complicated lower bound arguments previously used in analyses of non-elitist algorithms. We discuss this brief{}ly at the end of this section.

\begin{lemma}[Negative multiplicative drift theorem]\label{lem:nddrift}
  Let $X_0, X_1, \dots$ be a random process in a finite subset of $\R_{\ge 0}$. Assume that there are $\Delta, \delta > 0$ such that for each $t \ge 0$, the following multiplicative drift condition with additive disturbance holds:
  \begin{equation}\label{eq:nddrift}
  E[X_{t+1}] \le (1 - \delta) E[X_t] + \Delta. 
  \end{equation}
  Assume further that $E[X_0] \le \frac{\Delta}{\delta}$. Then the following two assertions hold.
  \begin{itemize}
  \item For all $t \ge 0$, $E[X_t] \le \frac{\Delta}{\delta}$.
  \item Let $M > \frac{\Delta}{\delta}$ and $T = \min\{t \ge 0 \mid X_t \ge M\}$. Then for all integers $L \ge 0$, 
  \[\Pr[T \ge L] \ge 1 - L \, \frac{\Delta}{\delta M},\]
  and $E[T] \ge \frac{\delta M}{2\Delta} - \frac 12$.
  \end{itemize}
\end{lemma}

The proof is an easy computation of expectations and an application of Markov's inequality similar to the direct proof of the multiplicative drift theorem in~\cite{DoerrG13algo}. We do not see a reason why the result should not also hold for processes taking more than a finite number of values, but since we are only interested in the finite setting, we spare us the more complicated world of continuous probability spaces.

\begin{proof}[Proof of Lemma~\ref{lem:nddrift}]
  If $E[X_t] \le \frac{\Delta}{\delta}$, then $E[X_{t+1}] \le (1-\delta) E[X_t] + \Delta \le {(1-\delta) \frac{\Delta}{\delta}} + \Delta = \frac{\Delta}{\delta}$ by~\eqref{eq:nddrift}. Hence the first claim follows by induction. To prove the second claim, we compute 
  \[\Pr[T < L] \le \Pr[X_0 + \dots + X_{L-1} \ge M] \le \frac{E[X_0 + \dots + X_{L-1}]}{M} \le \frac{L\Delta}{\delta M},\] where the middle inequality follows from Markov's inequality and the fact that the $X_t$ by assumption are all non-negative. From this estimate, using the shorthand $s = \lfloor \frac{\delta M}{\Delta} \rfloor$, we compute $E[T] = \sum_{t = 1}^\infty \Pr[T \ge t] \ge \sum_{t = 1}^{s} (1 - \frac{t\Delta}{\delta M}) = s - \frac 12 s (s+1) \frac{\Delta}{\delta M} \ge \frac{\delta M}{2\Delta} - \frac 12$, where the first equality is a standard way to express the expectation of a random variable taking non-negative integral values 
   and the last inequality is an elementary estimate that can be verified as follows. Let $\eps = \frac{\delta M}{\Delta} - s$. Then $s - \frac 12 s (s+1) \frac{\Delta}{\delta M} \ge \frac{\delta M}{2\Delta} - \frac 12$ is the same as $s - \frac{s (s+1)}{2(s+\eps)} \ge \frac 12 (s + \eps)  - \frac 12$, which is equivalent to 
   \[ 2s(s+\eps) -  s (s+1) + (s+\eps) \ge (s+\eps)^2\]
	since $s + \eps \ge 0$.  Now the left-hand side is equal to $s^2 + 2s\eps + \eps$, which is not smaller than $s^2 + 2s\eps + \eps^2$, since $\eps \in [0,1)$, and this is just the right-hand side.
\end{proof}

We 
note that in the typical application of this result (as in the proof of Theorem~\ref{thm:main} below), we expect to see the condition that for all $t \ge 0$, 
\begin{equation}\label{eq:ndplus}
  E[X_{t+1} \mid X_t] \le (1 - \delta) X_t + \Delta.
\end{equation}
Clearly, this condition implies~\eqref{eq:nddrift} by the law of total expectation.

We now argue that our negative multiplicative drift theorem is likely to find applications beyond ours to the negative drift in populations method in the following section. To this aim, we regard two classic lower bound analyses of non-elitist algorithms and point out where our drift theorem would have eased the analysis.

In~\cite{NeumannOW09}, Neumann, Oliveto, and Witt show that the variant of the simple genetic algorithm (simple GA) not using crossover needs time $2^{n^{1 - O(1 / \log\log n)}}$ to optimize the simple \onemax benchmark.
The key argument in~\cite{NeumannOW09} is as follows. The potential $X_t$ of the population $P^{(t)}$ in iteration $t$ is defined as $X_t = \sum_{x \in P^{(t)}} 8^{\onemax(x)}$. For this potential, it is shown~\cite[Lemma~7]{NeumannOW09} that if $X_t \ge 8^{0.996n}$, then $E[X_{t+1}] \le (1-\delta) X_t$ for some constant $\delta>0$. By bluntly estimating $E[X_{t+1}]$ in the case that $X_t < 8^{0.996n}$, this bound could easily be extended to $E[X_{t+1} | X_t] \le (1-\delta) X_t + \Delta$ for some number $\Delta$. This suffices to employ our negative drift theorem and obtain the desired lower bound. Without our drift theorem at hand, in~\cite{NeumannOW09} the potential $Y_t = \log_8(X_t)$ was considered, it was argued that it displays an additive drift away from the target and that $Y_t$ satisfies certain concentration statements necessary for the subsequent use of a negative drift theorem for additive drift. 

A second example where we feel that our drift theorem can ease the analysis is the work of Oliveto and Witt~\cite{OlivetoW14,OlivetoW15} on the simple GA with crossover optimizing \onemax. Due to the use of crossover, this work is much more involved, so shall not go into detail and simply point the reader to the location where negative drift occurs. In Lemma~19 of~\cite{OlivetoW15}, a multiplicative drift statement (away from the target) is proven. To use a negative drift theorem for additive drift (Theorem~2 in~\cite{OlivetoW15}), in the proof of Lemma~20 the logarithm of the original process is regarded. So here again, we think that a direct application of our drift theorem would have eased the analysis.

\section{Negative Drift in Populations Revisited}\label{sec:negpop}

In this section, we use our negative multiplicative drift result and some more arguments to rework Lehre's negative drift in populations method~\cite{Lehre10} and obtain Theorem~\ref{thm:main} further below. This method allows to analyze a broad class of evolutionary algorithms, namely all that can be described via the following type of population process.

\subsection{Population Selection-Mutation Processes}\label{ssec:psm}

A \emph{population selection-mutation (PSM) process} (called \emph{population selection-variation algorithm} in~\cite{Lehre10}) is the following type of random process. Let $\Omega$ be a finite set. We call $\Omega$ the \emph{search space} and its elements \emph{solution candidates} or \emph{individuals}. 
Let $\lambda \in \N$ be called the \emph{population size} of the process. An ordered multi-set of cardinality $\lambda$, in other words, a $\lambda$-tuple, over the search space $\Omega$ is called a \emph{population}. Let $\calP = \Omega^\lambda$ be the set of all populations. For $P \in \calP$, we write $P_1, \dots, P_\lambda$ to denote the elements of $P$. We also write $x \in P$ to denote that there is an $i \in [1..\lambda]$ such that $x = P_i$.

A PSM process starts with some, possibly random, population $P^{(0)}$. In each iteration $t = 1, 2, \dots$, a new population $P^{(t)}$ is generated from the previous one $P^{(t-1)}$ as follows. Via a (possibly) randomized \emph{selection operator} $\sel(\cdot)$, a $\lambda$-tuple of individuals is selected and then each of them creates an offspring through the application of a randomized \emph{mutation operator} $\mut(\cdot)$. 

The \emph{selection operator} can be arbitrary except that it only selects individuals from $P^{(t-1)}$. In particular, we do not assume that the selected individuals are independent. Formally speaking, the outcome of the selection process is a random $\lambda$-tuple $Q = \sel(P^{(t-1)}) \in [1..\lambda]^\lambda$ such that $P^{(t-1)}_{Q_1}, \dots, P^{(t-1)}_{Q_\lambda}$ are the selected parents.

From each selected parent $P^{(t-1)}_{Q_i}$, a single offspring $P^{(t)}_i$ is generated via a randomized \emph{mutation operator} $P^{(t)}_i = \mut(P^{(t-1)}_{Q_i})$. Formally speaking, for each $x \in \Omega$, $\mut(x)$ is a probability distribution on $\Omega$ and we write $y = \mut(x)$ to indicate that $y$ is sampled from this distribution. We assume that each sample, that is, each call of a mutation operator, uses independent randomness. With this notation, we can write the new population as $P^{(t)} = \big(\mathopen{}\mut(P^{(t-1)}_{Q_1}), \dots, \mut(P^{(t-1)}_{Q_\lambda})\big)$ with $Q = \sel(P^{(t-1)})$. 
From the definition it is clear that a PSM process is a Markov process with state space $\calP$. A pseudocode description of PSM processes is given in Algorithm~\ref{alg:psm}.

\begin{algorithm2e}%
	\For{$t = 1, 2, \ldots$}{
	  $(Q_1, \dots, Q_\lambda) \assign \sel(P^{(t-1)})$\;
	  $P^{(t)} \assign \big(\mathopen{}\mut(P^{(t-1)}_{Q_1}), \dots, \mut(P^{(t-1)}_{Q_\lambda})\big)$\;
  }
\caption{A PSM process with search space $\Omega$, population size $\lambda$, selection operator $\sel(\cdot)$ and mutation operator $\mut(\cdot)$, initialized with $P^{(0)} \in \Omega^\lambda$.}
\label{alg:psm}
\end{algorithm2e}
 
The following characteristic of the selection operator was found to be crucial for the analysis of PSM processes in~\cite{Lehre10}. Let $P \in \calP$ and $i \in [1..\lambda]$. Then the random variable $R(i,P) = |\{j \in [1..\lambda] \mid \sel(P)_j = P_i\}|$, called \emph{reproduction number of the $i$-th individual in $P$}, denotes the number of times $P_i$ was selected from $P$ as parent. Its expectation $E[R(i,P)]$ is called \emph{reproduction rate}.

\textbf{Example:} We now describe how the \mclea fits into this framework. That it fits into this framework and that the reproduction number is $\frac \lambda \mu$ was already stated in~\cite{Lehre10}, but how exactly this works out, to the best of our knowledge, was never made precise so far, and is also not totally trivial. 

We specify that when talking about the \mclea, given in pseudocode in Algorithm~\ref{alg:mclea}, we mean the basic EA which starts with a parent population of $\mu$ search points chosen independently and uniformly at random from $\{0,1\}^n$. In each iteration, $\lambda$ offspring are generated, each by selecting a parent individual uniformly at random (with repetition) and mutating it via standard bit mutation with mutation rate $p$. The next parent population is selected from these $\lambda$ offspring by taking $\mu$ best individuals, breaking ties randomly. 

\begin{algorithm2e}%
	Initialize $P^{(0)}$ with $\mu$ individuals chosen independently and uniformly at random from $\{0,1\}^n$\;
	\For{$t = 1, 2, \ldots$}{
    \For{$i \in [1..\lambda]$}{
      Select $x \in P^{(t-1)}$ uniformly at random\;
      Generate $y^{(i)}$ from $x$ via standard bit mutation\;
      }
		Select $P^{(t)}$ as sub-multiset of $\{y^{(1)}, \dots, y^{(\lambda)}\}$ with maximal fitness, breaking ties randomly\;
  }
\caption{The \mclea to maximize a function ${f : \{0,1\}^n \to \R}$.}
\label{alg:mclea}
\end{algorithm2e}

This algorithm can be modeled as a PSM process with population size $\lambda$ (not $\mu$). To do so, we need a slightly non-standard initialization of the population. We generate $P^{(0)}$ by first taking $\mu$ random search points and then generating each $P^{(0)}_i$, $i \in [1..\lambda]$, by choosing (with replacement) a random one of the $\mu$ base individuals and mutating it. With this definition, each individual in $P^{(0)}$ is uniformly distributed in $\{0,1\}^n$, but these individuals are not independent.

Given a population $P$ consisting of $\lambda$ individuals, the selection operator first selects a set $P_0$ of $\mu$ best individuals from $P$, breaking ties randomly. Formally speaking, this is a tuple $(i_1, \dots, i_\mu)$ of indices in $[1..\lambda]$. Then a random vector $(j_1, \dots, j_\lambda) \in [1..\mu]^\lambda$ is chosen and the selected parents are taken as $Q = (P_{i_{j_1}}, \dots, P_{i_{j_\lambda}})$. The next population $P'$ is obtained by applying the mutation operator to each of these, that is, $P' = (\mut(P_{i_{j_1}}), \dots, \mut(P_{i_{j_\lambda}}))$, where $\mut(\cdot)$ denotes standard bit mutation with mutation rate $p$.

From this description, it is clear that each individual of each population of the \mclea has a reproduction rate of $\frac {\lambda}{\mu}$.

\subsection{Our ``Negative Drift in Populations'' Result}

We prove the following version of the negative drift in populations method.

\begin{theorem}\label{thm:main}
  Consider a PSM process $(P^{(t)})_{t \ge 0}$ with associated reproduction numbers $R(\cdot,\cdot)$ as defined in Section~\ref{ssec:psm}. Let $g : \Omega \to \Z_{\ge 0}$, called potential function, and $a, b \in \Z_{\ge 0}$ with $a \le b$. Assume that for all $x \in P^{(0)}$ we have $g(x) \ge b$. Let $T = \min\{t \ge 0 \mid \exists i \in [1..\lambda] : g(P^{(t)}_i) \le a\}$ the first time we have a search point with potential $a$ or less in the population. Assume that the following three conditions are satisfied.
  \begin{enumerate}
  \item There is an $\alpha \ge 1$ such that for all populations $P \in \calP$ with $\min\{g(P_i) \mid i \in [1..\lambda]\} > a$ and all $i \in [1..\lambda]$ with $g(P_i) < b$, we have $E[R(i,P)] \le \alpha$.
  \item There is a $\kappa > 0$ and a $0 < \delta < 1$ such that for all $x \in \Omega$ with $a < g(x) < b$ we have \[E[\exp(-\kappa g(\mut(x)))] \le \frac 1 \alpha (1 - \delta) \exp(-\kappa g(x)).\]
  \item There is a $D \ge \delta$ such for all $x \in \Omega$ with $g(x) \ge b$, we have 
  \[E[\exp(-\kappa g(\mut(x)))] \le D \exp(-\kappa b).\]  
  \end{enumerate}
  Then 
  \begin{itemize}
  \item $E[T] \ge \frac{\delta}{2D \lambda} \exp(\kappa(b-a)) - \frac 12$, and 
  \item for all $L \ge 1$, we have $\Pr[T < L] \le L \lambda \frac{D}{\delta} \exp(-\kappa (b-a))$.
  \end{itemize}
\end{theorem}

Before proceeding with the proof, we compare our result with Theorem~1 of~\cite{Lehre10}. We first note that, apart from a technicality which we discuss toward the end of this comparison, the assumptions of our result are weaker than the ones in~\cite{Lehre10} since we do not need the technical fourth and fifth assumption of~\cite{Lehre10}, which in our notation would read as follows.
\begin{itemize}
\item There is a $\delta_2 > 0$ such that for all $i \in [a..b]$ and all $k, \ell \in \Z$ with $1 \le k + \ell$ and all $x, y \in \Omega$ with $g(x) = i$ and $g(y) = i-\ell$ we have 
\begin{align*}
&\Pr[g(\mut(x)) = i - \ell \wedge g(\mut(y)) = i - \ell - k] \\
&\quad\le \exp(\kappa (1-\delta_2) (b-a)) \Pr[g(\mut(x)) = i -k - \ell].
\end{align*}
\item There is a $\delta_3 > 0$ such that for all $i, j, k, \ell \in \Z$ with $a \le i \le b$ and $1 \le k + \ell \le j$ and all $x, y \in \Omega$ with $g(x) = i$ and $g(y) = i - k$ we have 
\[\Pr[g(\mut(x)) = i - j] \le \delta_3 \Pr[g(\mut(y)) = i - k - \ell].\]
\end{itemize} 
The assertion of our result is of the same type as in~\cite{Lehre10}, but stronger in terms of numbers. For the probability $\Pr[T < L]$ to find a potential of at most $a$ in time less than $L$, a bound of \[{O(\lambda L^2 D \,(b-a) \exp(-\kappa \delta_2 (b-a)))}\] is shown in~\cite{Lehre10}. Hence our result is smaller by a factor of ${\Omega(L (b-a) \exp(-\kappa (1-\delta_2) (b-a))}$. In addition, our result is non-asymptotic, that is, the lower bound contains no asymptotic notation or unspecified constants. 

The one point where Lehre's~\cite{Lehre10} result potentially is stronger is that it needs assumptions only on the ``average drift'' from the random search point at time $t$ conditional on having a fixed potential, whereas we require the same bound on the ``point-wise drift'', that is, conditional on the current search point being equal to a particular search point of this potential. Let us make this more precise.  Lehre uses the notation $(X_t)_{t \ge 0}$ to denote the Markov process on $\Omega$ associated with the mutation operator (unfortunately, it is not said in~\cite{Lehre10} what is $X_0$, that is, how this process is started). Then $\Delta_t(i) = (g(X_{t+1} - g(X_t) \mid g(X_t) = i)$ defines the potential gain in step $t$ when the current state has potential $i$. With this notation, instead of our second and third condition, Lehre~\cite{Lehre10} requires only the weaker conditions (here again translated into our notation).
\begin{itemize}
\item[(ii')] For all $t \ge 0$ and all $a < i < b$, $E[\exp(-\kappa \Delta_t(i))] < \frac 1\alpha (1-\delta)$.
\item[(iii')] For all $t \ge 0$, $E[\exp(-\kappa (g(X_{t+1}) - b)) \mid g(X_t) \ge b] < D$.
\end{itemize}

So Lehre only requires that the random individual at time $t$, conditional on having a certain potential, gives rise to a certain drift, whereas we require that each particular individual with this potential gives rise to this drift. On the formal level, Lehre's condition is much weaker than ours (assuming that the unclear point of what is $X_0$ can be fixed). That said, to exploit such weaker conditions, one would need to be able to compute such average drifts and they would need to be smaller than the worst-case point-wise drift. We are not aware of many examples where average drift was successfully used in drift analysis (one is J\"agersk\"upper's remarkable analysis of the linear functions problem~\cite{Jagerskupper08}) despite the fact that many classic drift theorems only require conditions on the average drift to hold.

We now prove Theorem~\ref{thm:main}. Before stating the formal proof, we describe on a high level its main ingredients and how it differs from Lehre's proof. 

The main challenge when using drift analysis is designing a potential function that suitably measures the progress. For simple hillclimbers and optimization problems, the fitness of the current solution may suffice, but already the analysis of the \oea on linear functions resisted such easy approaches~\cite{HeY01,DrosteJW02,DoerrJW12algo,Witt13}. For population-based algorithms, the additional challenge is to capture the quality of the whole population in a single number. We note at this point that the notion of ``negative drift in populations'' was used in Lehre to informally describe the characteristic of the population processes regarded, but drift analysis as a mathematical tool was employed only on the level of single individuals and the resulting findings were lifted to the whole population via advanced tools like branching processes and eigenvalue arguments. 

To prove upper bounds, in~\cite{Witt06,ChenHSCY09,Lehre11,DangL16algo,CorusDEL18,AntipovDFH18,DoerrK19}, implicitly or explicitly potential functions were used that build on the fitness of the best individual in the population and the number of individuals having this fitness. Regarding only the current-best individuals, these potential functions might not be suitable for lower bound proofs. 

The lower bound proofs in~\cite{NeumannOW09,OlivetoW14,OlivetoW15,AntipovDY19} all define a natural potential for single individuals, namely the Hamming distance to the optimum, and then lift this potential to populations by summing over all individuals an exponential transformation of their base potential (this ingenious definition was, to the best of our knowledge, not known in the theory of evolutionary algorithms before the work of Neumann, Oliveto, and Witt~\cite{NeumannOW09}). This is the type of potential we shall use as well, and given the assumptions of Theorem~\ref{thm:main}, it is not surprising that $\sum_{x \in P} \exp(-\kappa g(x))$ is a good choice. For this potential, we shall then show with only mild effort that it satisfies the assumptions of our drift theorem, which yields the desired lower bounds on the runtime (using that a single good solution in the population already requires a very high potential due to the exponential scaling). We now give the details of this proof idea.

\begin{proof}[Proof of Theorem~\ref{thm:main}]
  We consider the process $(X_t)_{t \geq 0}$ defined by $X_t = \sum_{i=1}^\lambda \exp(-\kappa g(P^{(t)}_i))$. To apply drift arguments, we first analyze the expected state after one iteration, that is, $E[X_t \mid X_{t-1}]$. To this end, let us consider a fixed parent population $P = P^{(t-1)}$ in iteration $t$. Let $Q = \sel(P)$ be the indices of the individuals selected for generating offspring. 
  
  We first condition on $Q$ (and as always on $P$), that is, we regard only the probability space defined via the mutation operator, and compute
  \begin{align*}
  E[X_{t} \mid Q] & = E\left[\sum_{j = 1}^\lambda \exp(-\kappa g(\mut(P_{Q_j})))\right] \\
  & = \sum_{i = 1}^\lambda (R(i,P) \mid Q) E[\exp(-\kappa g(\mut(P_i)))].
  \end{align*}
  Not anymore conditioning on $Q$, using the law of total expectation, using the assumptions~(ii) and~(iii) on the drift from mutation, and finally using assumption~(i) on the reproduction number and the trivial fact that $\sum_{i=1}^\lambda R(i,P) = \lambda$, we have
  \begin{align*}
  E[X_t] & = E_Q[E[X_t \mid Q]] \\
  & = \sum_{i = 1}^\lambda E[R(i,P)] E[\exp(-\kappa g(\mut(P_i)))]\\
  &\le \sum_{i : g(P_i) < b} E[R(i,P)] \tfrac 1 \alpha (1-\delta) \exp(-\kappa g(P_i)) \\ 
  & \quad\, + \sum_{i : g(P_i) \ge b} E[R(i,P)] D \exp(-\kappa b)\\
  & \le \sum_{P_i : g(P_i) < b} \alpha \cdot \tfrac 1 \alpha (1-\delta) \exp(-\kappa g(P_i))
   + \lambda \cdot D \exp(-\kappa b)\\
  & \le (1-\delta) X_{t-1} + \lambda D \exp(-\kappa b)  
\end{align*}
and recall that this is conditional on $P^{(t-1)}$, hence also on $X_{t-1}$. 

Let $\Delta = \lambda D \exp(-\kappa b)$. Since $P^{(0)}$ contains no individual with potential below~$b$, we have $X_0 \le \lambda \exp(-\kappa b) = \frac{\Delta}{D} \le \frac{\Delta}{\delta}$. Hence also the assumption $E[X_0] \le \frac{\Delta}{\delta}$ of Lemma~\ref{lem:nddrift} is fulfilled. 

Let $M = \exp(-\kappa a)$ and $T' := \min\{t \ge 0 \mid X_t \ge M\}$. Note that $T$, the first time to have an individual with potential at most $a$ in the population, is at least $T'$. Now the negative multiplicative drift theorem (Lemma~\ref{lem:nddrift}) gives 
\begin{align*}
  &\Pr[T < L] \le \Pr[T' < L] \le \frac{L\Delta}{M\delta} = L \lambda D \frac{\exp(-\kappa (b-a))}{\delta}, \\
  &E[T] \ge E[T'] \ge \frac{\delta M}{2 \Delta} - \frac 12 = \frac{\delta}{2D \lambda} \exp(\kappa(b-a)) - \frac 12.\\&\qedhere
\end{align*} 
\end{proof}

We note that the proof above actually shows the following slightly stronger statement, which can be useful when working with random initial populations (as, e.g., in the following section).

\begin{theorem}\label{thm:main2}
  Theorem~\ref{thm:main} remains valid when the assumption that all initial individuals have potential at least $b$ is replaced by the assumption $\sum_{i=1}^\lambda E[\exp(-\kappa g(P^{(0)}_i))] \le \frac{\lambda D \exp(-\kappa b)}{\delta}$.
\end{theorem}

\section{Processes Using Standard Bit Mutation}

Since many EAs use standard bit mutation, as in~\cite{Lehre10} we now simplify our main result for processes using standard bit mutation and for $g$ being the Hamming distance to a target solution. Hence in this section, we have $\Omega = \{0,1\}^n$ and $y = \mut(x)$ is obtained from $x$ by flipping each bit of $x$ independently with probability $p$. Since our results are non-asymptotic, we can work with any $p \le \frac 12$. 

\begin{theorem}\label{thm:sbm}
  Consider a PSM process (see Section~\ref{ssec:psm}) with search space $\Omega = \{0,1\}^n$, using standard bit mutation with mutation rate $p \in [0,\frac 12]$ as mutation operator, and such that $P^{(0)}_i$ is uniformly distributed in $\Omega$ for each $i \in [1..\lambda]$ (possibly with dependencies among the individuals). Let $x^* \in \Omega$ be the target of the process. For all $x \in \Omega$, let $g(x) := H(x,x^*)$ denote the Hamming distance from the target. 
  
  Let $\alpha \ge 1$ and $0 < \delta < 1$ such that $\ln(\frac{\alpha}{1-\delta}) < pn$, that is, such that $1 - \frac{1}{pn}\ln(\frac{\alpha}{1-\delta}) =: \eps > 0$. Let $B = \frac{2}{\eps}$. Let $a, b$ be integers such that $0 \le a < b$ and $b \le \tilde b := n \frac{1}{B^2 - 1}$. 

 Selection condition: Assume that for all populations $P \in \calP$ with $\min\{g(P_i) \mid i \in [1..\lambda]\} > a$ and all $i \in [1..\lambda]$ with $g(P_i) < b$, we have $E[R(i,P)] \le \alpha$.
 
 Then the first time $T := \min\{t \ge 0 \mid \exists i \in [1..\lambda] : g(P^{(t)}_i) \le a\}$ that the population contains an individual in distance $a$ or less from $x^*$ satisfies  
  \begin{align*}
  E[T] &\ge \frac{1}{2\lambda} \min\left\{\frac{\delta \alpha}{1-\delta}, 1\right\} \exp\left(\ln\left(\frac{2}{1 - \frac{1}{pn} \ln(\frac{\alpha}{1-\delta})}\right) (b-a) \right)  - \frac 12,\\
  \Pr[T < L] & \le L \lambda \max\left\{\frac{1-\delta}{\delta \alpha}, 1\right\} \exp\left(- \ln\left(\frac{2}{1 - \frac{1}{pn} \ln(\frac{\alpha}{1-\delta})}\right) (b-a) \right).
  \end{align*}
\end{theorem}

The proof of this result is a reduction to Theorem~\ref{thm:main}. To show that the second and third condition of Theorem~\ref{thm:main} are satisfied, one has to estimate $E[\exp(-\kappa (g(\mut(x)) - g(x)))]$, which is not difficult since $g(\mut(x))-g(x)$ can be written as sum of independent random variables. With a similar computation and some elementary calculus, we show that the weaker starting condition of Theorem~\ref{thm:main2} is satisfied.

\begin{proof}[Proof of Theorem~\ref{thm:sbm}]
  We apply Theorem~\ref{thm:main}. 
  To show the second and third condition of the theorem, let $x \in \Omega$ and let $y = \mut(x)$ be the random offspring generated from $x$. We use the shorthand $d = g(x)$. We note that $g(y) - g(x) = g(y) - d$ can be expressed as a sum of $n$ independent random variables $Z_1, \dots, Z_n$ such that for $i \in [1..d]$, we have $\Pr[Z_i = -1] = p$ and $\Pr[Z_i = 0] = 1-p$, and for $i = [d+1..n]$, we have $\Pr[Z_i = +1] = p$ and $\Pr[Z_i =0] = 1-p$. 
  
  Let $\kappa \ge 0$ be arbitrary for the moment. We note that for $i \in [1..d]$, we have $E[\exp(-\kappa Z_i)] = (1-p)\cdot 1 + p e^{\kappa} = 1 + p(e^\kappa - 1)$ and for $i = [d+1..n]$, analogously, $E[\exp(-\kappa Z_i)] = (1-p)\cdot 1 + p e^{-\kappa} = 1 - p(1 - e^{-\kappa})$ (formally speaking, we compute here the moment-generating function of a Bernoulli random variable). Using the independence of the $Z_i$, these elementary arguments, and the standard estimate $1+r \le \exp(r)$, we compute
  \begin{align}
  E[\exp(-&\kappa(g(y)-g(x))] 
  = E\left[\prod_{i=1}^n \exp(-\kappa Z_i)\right]
  = \prod_{i=1}^n E[\exp(-\kappa Z_i)]\label{eq:mfgpseudobin}\\
  &= (1  + p(e^\kappa - 1))^d (1 - p(1 - e^{-\kappa}))^{n-d}\nonumber\\
  &\le \exp(dp(e^\kappa - 1)) \cdot \exp(-(n-d)p(1 - e^{-\kappa}))\nonumber\\
  &= \exp(dpe^\kappa + (n-d)pe^{-\kappa} - pn).\nonumber
  \end{align} 

  Let now $\kappa = \ln(B)$. We consider first the case that $d \le b$, which implies $d \le \tilde b$. We continue the above computation via
  \begin{align}
  E&[\exp(-\kappa(g(y)-g(x))] 
  \le \exp(\tilde b p B + (n-\tilde b) p \tfrac 1B - pn)\nonumber\\
  & = \exp\left(pn \left(\frac{B}{B^2 -1} + \left(1-\frac{1}{B^2-1}\right) \frac 1B - 1\right)\right)
   = \exp\left(pn(-1+\tfrac 2B)\right) \label{eq:sbmnontrivial}\\
  & = \exp\left(pn \left(-\frac {1}{pn} \ln\left(\frac{\alpha}{1-\delta}\right)\right)\right)
   = (1-\delta) \frac 1 \alpha.\nonumber
  \end{align} 
  This shows the second condition of Theorem~\ref{thm:main} for $\kappa = \ln(B)$. 
  
  To show that the third condition of Theorem~\ref{thm:main} is satisfied, assume that $g(x) \ge b$. We first note the following. Let $x' \in \Omega$ with $g(x') = b$ and let $y' = \mut(x')$. By Lemma~\ref{lem:witt}, $g(y)$ stochastically dominates $g(y')$. Consequently, by Lemma~\ref{lem:mondom},
  \begin{align*}
    E[\exp(-\kappa(g(y)-b))] 
  & \le E[\exp(-\kappa(g(y')-b))] \\
  & = E[\exp(-\kappa(g(y')-g(x'))]  \le (1-\delta) \frac 1 \alpha,
  \end{align*}
  where the last estimate exploits that we have shown the second condition also for $g(x) = b$. Hence with $D = \max\{(1-\delta) \frac 1 \alpha, \delta\}$ we have also shown the third condition of Theorem~\ref{thm:main} (including the requirement $D \ge \delta$).
  
  We finally show that the starting condition in Theorem~\ref{thm:main2} is satisfied. Using the moment-generating function of a binomially distributed random variable (which is nothing more than the arguments used in~\eqref{eq:mfgpseudobin}), this follows immediately from the following estimate, valid for a random search point $x$: 
  \begin{align*}
  E[\exp(-\kappa g(x))] & = (\tfrac 12 + \tfrac 12 \exp(-\kappa))^n \le \exp(-\kappa n / (B^2 -1))\\
  &\le  \exp(-\kappa b) \le \tfrac{D}{\delta} \exp(-\kappa b).
  \end{align*}
  The estimate above is easy to see apart from the first inequality, which requires some elementary calculus. Recalling $\kappa=\ln(B)$, this inequality is equivalent to $\frac 12 + \frac 1 {2B} \le B^{-1/(B^2-1)}$. The latter is satisfied for $B=2$. Since its left-hand side is decreasing in $B$, we now show that the right-hand side is increasing in $B$ and obtain that the inequality is satisfied for all $B \ge 2$ (and we note that always $B \ge 2$ since $\eps \le 1$). By the monotonicity of the logarithm, the function $B \mapsto B^{-1/(B^2-1)}$ is increasing (in $\R_{>0}$) if and only if $B \mapsto \ln(B^{-1/(B^2-1)}) = - \frac{\ln B}{B^2-1}$ is increasing, which is easily seen to be true by noting that its derivative $B \mapsto \frac{B^2 (2 \ln(B) - 1) + 1}{B(B^2-1)^2}$ is positive for $B \ge 2$. 
  
     Consequently, the random initial population $P^{(0)}$ satisfies 
   \[\sum_{i=1}^\lambda E[\exp(-\kappa g(P^{(0)}_i))] \le \frac{\lambda D \exp(-\kappa b)}{\delta}\]
    as required in Theorem~\ref{thm:main2}. From the conclusion of Theorem~\ref{thm:main}, we obtain
  \begin{align*}
  E[T] &\ge \frac{\delta}{2D\lambda} \exp(\kappa (b-a)) - \frac 12 \\
  &= \frac{1}{2\lambda} \min\left\{\frac{\delta \alpha}{1-\delta}, 1\right\} \exp\left(\ln\left(\frac{2}{1 - \frac{1}{pn} \ln(\frac{\alpha}{1-\delta})}\right) (b-a) \right) - \frac 12,\\
  \Pr[T < L] & \le L \lambda \tfrac{D}{\delta} \exp(-\kappa(b-a)) \\
  & = L \lambda \max\left\{\frac{1-\delta}{\delta \alpha}, 1\right\} \exp\left(- \ln\left(\frac{2}{1 - \frac{1}{pn} \ln(\frac{\alpha}{1-\delta})}\right) (b-a) \right).
  \end{align*}
\end{proof}

As a simple \textbf{example} for an application of this result, let us consider the classic $(\mu,\lambda)$ EA (with uniform selection for variation, truncation selection for inclusion into the next generation, and mutation rate $p = \frac 1n$) with $\lambda = 2\mu$ optimizing some function $f : \{0,1\}^n \to \R$, $n = 500$, with unique global optimum. For simplicity, let us take as performance measure $\lambda T$, that is, the number of fitness evaluations in all iterations up to the one in which the optimum was found.
Since $\lambda = 2\mu$, we have $\alpha = 2$. By taking $\delta = 0.01$, we obtain a concrete lower bound of an expected number of more than 13~million fitness evaluations until the optimum is found (regardless of $\mu$ and~$f$). 

Since Theorem~\ref{thm:sbm} is slightly technical, we now formulate the following corollary, which removes the variable $\delta$ without significantly weakening the result.
We note that the proof of this result applies Theorem~\ref{thm:sbm} with a non-constant~$\delta$, so we do not see how such a result could have been proven from Lehre's result~\cite{Lehre10}.

\begin{corollary}\label{cor:sbm}
  Consider a PSM process as in Theorem~\ref{thm:sbm}. Let $x^* \in \Omega$ be the target of the process. For all $x \in \Omega$, let $g(x) := H(x,x^*)$ denote the Hamming distance from the target. Assume that there is an $\alpha \ge 1$ such that 
  \begin{itemize}
  \item $\ln(\alpha) \le p(n-1)$, which is equivalent to $\gamma := 1 - \frac{\ln \alpha}{pn} \ge \frac 1n$;
  \item there is an $a \le b := \lfloor (1 - \tfrac 4n) n \frac{1}{\frac{4}{\gamma^2}-1} \rfloor$ such that for all populations $P \in \calP$ with $\min\{g(P_i) \mid i \in [1..\lambda]\} > a$ and for all $i \in [1..\lambda]$, we have $E[R(i,P)] \le \alpha$.
  \end{itemize}
 
 Then the first time $T := \min\{t \ge 0 \mid \exists i \in [1..\lambda] : g(P^{(t)}_i) \le a\}$ that the population contains an individual in distance $a$ or less from $x^*$ satisfies  
\begin{align*}
  E[T]    &\ge \frac{p \alpha}{4\lambda n} \min\left\{1,\frac{2n}{p\alpha}\right\} \exp\left(\ln\left(\frac{2}{\gamma}\right) \left(b-a \right) \right) - \frac 12,\\
  \Pr[T < L]   & \le \frac{2 L \lambda n}{p \alpha} \max\left\{1,\frac{p\alpha}{2n}\right\} \exp\left(- \ln\left(\frac{2}{\gamma}\right) \left(b-a\right) \right).
  \end{align*}
  In particular, if $a \le (1-\eps)b$ for some constant $\eps > 0$, then $T \lambda$ is super-polynomial in $n$ (in expectation and with high probability) when $\gamma = \omega(n^{-1/2})$ and at least exponential when $\gamma = \Omega(1)$.
\end{corollary}

The main argument is employing Theorem~\ref{thm:sbm} with the $\delta = \frac p{2n}$ and computing that this small $\delta$ has no significant influence on the exponential term of the bounds.

\begin{proof}[Proof of Corollary~\ref{cor:sbm}]  
  We apply Theorem~\ref{thm:sbm} with $\delta = \frac p{2n}$. Since $\delta \le \frac 12$, we have $1-\delta \ge \exp(-2\delta)$ and thus $\ln(\frac{\alpha}{1-\delta}) = \ln(\alpha) - \ln(1-\delta) \le \ln(\alpha) + 2\delta$. Consequently, $\eps := 1 - \frac{1}{pn} \ln(\frac{\alpha}{1-\delta})$ defined as in Theorem~\ref{thm:sbm} satisfies 
  \[\eps  \ge 1 - \tfrac{1}{pn} (\ln(\alpha) + 2\delta) = 1 - \tfrac{\ln \alpha}{pn} - \tfrac 1{n^2} \ge (1 - \tfrac{\ln \alpha}{pn}) (1 - \tfrac 1n) = \gamma (1 - \tfrac 1n),\] 
   where the second inequality uses our assumption $\ln \alpha \le p(n-1)$. Now 
  \begin{align*}
  \tilde b &:= n \frac{1}{\frac 4{\eps^2} -1} 
  \ge n \frac{1}{\frac 4{\gamma^2 (1 - \frac 1n)^2} -1} 
  \ge n \frac{1}{\frac {4 - \gamma^2 (1 - \frac 2n)}{\gamma^2 (1 - \frac 2n)}} 
  =  n \frac{\gamma^2 (1 - \frac 2n)}{4 - \gamma^2 + \gamma^2 \frac 2n} \\
  & \ge n \frac{\gamma^2 (1 - \frac 2n)}{4 - \gamma^2 + (4 - \gamma^2) \frac 2n} 
  =  n \frac{\gamma^2 (1 - \frac 2n)}{(4 - \gamma^2) (1 + \frac 2n)} 
  \ge n \frac{\gamma^2 (1 - \frac 2n)^2}{4 - \gamma^2} \\
  & \ge (1 - \tfrac 4n) n \frac{1}{\frac{4}{\gamma^2} - 1}.
  \end{align*}
  
  With these estimates, $b \le \lfloor \tilde b \rfloor$, and the definition of $\delta$, the bounds of Theorem~\ref{thm:sbm} become
\begin{align*}
  E[T] 
  &\ge \frac{1}{2\lambda} \min\left\{\frac{\delta \alpha}{(1-\delta)}, 1\right\} \exp\left(\ln\left(\frac{2}{1 - \frac{1}{pn} \ln(\frac{\alpha}{1-\delta})}\right) (b-a) \right) - \frac 12\\
  &\ge \frac{p \alpha}{4\lambda n} \min\left\{1,\frac{2n}{p\alpha}\right\} \exp\left(\ln\left(\frac{2}{\gamma}\right) \left(b-a \right) \right) - \frac 12,\\
  \Pr[T < L] 
  & \le L \lambda \max\left\{\frac{(1-\delta)}{\delta \alpha}, 1\right\} \exp\left(- \ln\left(\frac{2}{1 - \frac{1}{pn} \ln(\frac{\alpha}{1-\delta})}\right) (b-a) \right)\\
  & \le \frac{2 L \lambda n}{p \alpha} \max\left\{1,\frac{p\alpha}{2n}\right\} \exp\left(- \ln\left(\frac{2}{\gamma}\right) \left(b-a\right) \right).
  \end{align*}
	For the asymptotic statements, we observe first that $\frac{p\alpha}{4n}\min\{1,\frac{2n}{p\alpha}\} = \min\{\frac{p\alpha}{4n},\frac 12\} \ge \min\{\frac{\alpha \ln(\alpha)}{4n(n-1)},\frac 12\}$ since $p \ge \ln(\alpha)/(n-1)$ due to our assumption that $\ln(a)\le p(n-1)$. Hence $E[T]\lambda$ is super-polynomial or at least exponential if and only if the term $\exp(\ln(2/\gamma)(b-a))$ is. So it suffices to regard the latter term. 
	
	We note that $b = \Theta(n\gamma^2)$ since $\gamma$ is always at most one. By assumption, $(b-a) = \Theta(b)$.  Assume first that $\gamma = \omega(n^{-1/2})$. If $\gamma \le n^{-1/4}$, then $\exp(\ln(2/\gamma)(b-a)) = (2/\gamma)^{b-a} \ge (2n^{1/4})^{\omega(1)}$, which is super-polynomial. If $\gamma \ge n^{-1/4}$, then $b-a = \Omega(n^{1/2})$ and $\exp(\ln(2/\gamma)(b-a)) \ge 2^{b-a}$ is again super-polynomial. This shows the claimed super-polynomiality for $\gamma = \omega(n^{-1/2})$. 
	
	For $\gamma = \Omega(1)$, we have $b-a = \Theta(b) = \Theta(n)$ and thus $\exp(\ln(2/\gamma)(b-a)) \ge \exp(\ln(2)(b-a)) = \exp(\Theta(n))$ is exponential in $n$.  
	
	The asymptotic statements of the with-high-probability claims follow analogously.
\end{proof}

\section{Standard Bit Mutation with Random Mutation Rate}\label{sec:randsbm}

To analyze a \emph{uniform mixing hyper-heuristic} which uses standard-bit mutation with a mutation rate randomly chosen from a finite set of alternatives, Dang and Lehre~\cite[Theorem~2]{DangL16ppsn} extend Theorem~\ref{thm:sbm} to such mutation operators. They do not give a proof of their result, stating that it would be similar to the proof of the result for classic standard bit mutation \cite[Theorem~4]{Lehre10}. Since we did not find this so obvious, we reprove this result now with our methods. The non-asymptoticity of our result allows to extend it to super-constant numbers of mutation rates and to mutation rates other than $\Theta(1/n)$. We note that such situations appear naturally with the heavy-tailed mutation operator proposed in~\cite{DoerrLMN17}. 

We show the following result, which extends Theorem~\ref{thm:sbm}.

\begin{theorem}\label{thm:randsbm}
  Let $n \in \N$. Let $m \in \N$, $p_1, \dots, p_m \in [0,\frac 12]$, and $q_1, \dots, q_m \in [0,1]$ such that $\sum_{i=1}^m q_i = 1$. Let $\mut$ be the mutation operator which, in each application independently, chooses an $I \in [1..m]$ with probability $\Pr[I = i] = q_i$ for all $i \in [1..m]$ and then applies standard bit mutation with mutation rate $p_I$. 
  
  Consider a PSM process (see Section~\ref{ssec:psm}) with search space $\Omega = \{0,1\}^n$, using this mutation operator $\mut(\cdot)$, and such that each initial individual is uniformly distributed in $\Omega$ (not necessarily independently). Let $x^* \in \Omega$ be the target of the process. For all $x \in \Omega$, let $g(x) := H(x,x^*)$ denote the Hamming distance from the target. 
  
  Let $\alpha \ge 1$, $0 < \delta < 1$, and $B > 2$ such that 
  \begin{equation}
  \sum_{i=1}^m q_i \exp(- p_i n (1 - \tfrac 2B)) \le (1-\delta) \, \frac 1 \alpha. \label{eq:bbb}
  \end{equation}
Let $a, b$ be integers such that $0 \le a < b \le \tilde b := n \frac{1}{B^2 - 1}$. 

 Selection condition: Assume that for all populations $P \in \calP$ with $\min\{g(P_i) \mid i \in [1..\lambda]\} > a$ and all $i \in [1..\lambda]$ with $g(P_i) < b$, we have $E[R(i,P)] \le \alpha$.
 
 Then the first time $T := \min\{t \ge 0 \mid \exists i \in [1..\lambda] : g(P^{(t)}_i) \le a\}$ that the population contains an individual in distance $a$ or less from $x^*$ satisfies  
  \begin{align*}
  E[T] &\ge \frac{1}{2\lambda} \min\left\{\frac{\delta \alpha}{(1-\delta)}, 1\right\} \exp\left(\ln(B) (b-a) \right)  - \frac 12,\\
  \Pr[T < L] & \le L \lambda \max\left\{\frac{(1-\delta)}{\delta \alpha}, 1\right\} \exp\left(- \ln(B)) (b-a) \right).
  \end{align*}
\end{theorem}

It is clear that when using standard bit mutation with a random mutation rate, then the drift -- regardless of whether we just regard the fitness or an exponential transformation of it -- is a convex combination of the drift values of each of the individual mutation operators. The reason why this argument does not immediately extend Theorem~\ref{thm:sbm} to random mutation rates is that the mutation rate also occurs in the exponential term $\exp(pn(-1+\frac 2B))$ in equation~\eqref{eq:sbmnontrivial}. Apart from this difficulty, however, we can reuse large parts of the proof of Theorem~\ref{thm:sbm}. 

\begin{proof}[Proof of Theorem~\ref{thm:randsbm}]
  Let $\kappa := \ln(B)$. Let $x \in \Omega$, $d := g(x)$, and $y = \mut(x)$. Assuming $d \le b$, analogous to the proof of Theorem~\ref{thm:sbm}, we have 
  \begin{align*}
  E&[\exp(-\kappa(g(y) - g(x))] \\
  & = \sum_{i=1}^m q_i (1  + p_i(e^\kappa - 1))^d (1 - p_i(1 - e^{-\kappa}))^{n-d}\\
  & \le \sum_{i=1}^m q_i \exp(p_i n(-1+\tfrac 2B)) \le (1-\delta) \, \frac 1 \alpha.
  \end{align*}
  This shows the second condition of Theorem~\ref{thm:main} and, with the same domination argument as in the proof of Theorem~\ref{thm:sbm} and $D = \max\{(1-\delta)\frac 1 \alpha, \delta\}$, also the third condition of Theorem~\ref{thm:main}. The starting condition of Theorem~\ref{thm:main2} follows as in the proof of Theorem~\ref{thm:sbm} by noting that again $B \ge 2$. Now Theorem~\ref{thm:main} is applicable and as in the last few lines of the proof of Theorem~\ref{thm:sbm} we show our claim (note that now it suffices to simply replace $\kappa$ by $\ln B$ and not by the more complicated logarithmic term there).
\end{proof}

Equation~\eqref{eq:bbb} defining the admissible values for $B$ and thus for the starting point $b$ of the negative drift regime is not very convenient to work with in general. We stated it nevertheless because in particular situations it might be useful, e.g., to show an inapproximability result, that is, that a certain algorithm cannot come closer to the optimum than by a certain margin in subexponential time. The following weaker assumption is easier to work with and should, in most cases, give satisfactory results as well.

\begin{lemma}\label{lem:randsbm}
  Assume that in the notation of Theorem~\ref{thm:randsbm}, we have 
  \[\sum_{i=1}^m q_i \exp(-p_i n) \le (1 - \gamma) \frac 1\alpha\]
  for some $0 < \gamma < 1$. Then~\eqref{eq:bbb} and $B > 2$ are satisfied for $\delta = \frac 12 \gamma$ and 
  \[B = 2 \left(1 - \frac{\ln \frac{1 - \gamma/2}{\alpha}}{\ln \frac{1 - \gamma}{\alpha}}\right)^{-1}.\]
\end{lemma}

\begin{proof}
  Let $\eps = 1 - \frac{\ln \frac{1 - \gamma/2}{\alpha}}{\ln \frac{1 - \gamma}{\alpha}}$ so that $\frac 2B = \eps$. We note that $0 < \eps < 1$ and thus $B > 2$. By the concavity of the exponentiation with numbers smaller than one, we have 
  \begin{align*}
  \sum_{i=1}^m q_i \exp(- p_i n (1 - \tfrac 2B)) 
  & \le \left(\sum_{i=1}^m q_i \exp(- p_i n)\right)^{1-\eps}\\
  & \le \bigg(\frac{1-\gamma}{\alpha}\bigg)^{1-\eps} = \bigg(1 - \frac \gamma 2\bigg) \frac 1 \alpha.
  \end{align*}
\end{proof}
  
If in an asymptotic setting $\gamma$ and $\alpha$ can be taken as constants, then this lemma and Theorem~\ref{thm:randsbm} show an exponential lower bound on the runtime. This proves~\cite[Theorem~2]{DangL16ppsn} and extends it to mutation rates that are not necessarily $\Theta(1/n)$. 

As an example where mutation rates other than $\Theta(1/n)$ occur, we now regard the heavy-tailed mutation operator proposed in~\cite{DoerrLMN17}. This operator was shown to give a uniformly good performance of the \oea on all jump functions, whereas each fixed mutation rate was seen to be good only for a small range of jump sizes. The heavy-tailed operator and variations of it have shown a good performance also in other works, e.g.,~\cite{MironovichB17,FriedrichQW18,FriedrichGQW18,FriedrichGQW18heavysubm,WuQT18,AntipovBD20gecco,AntipovBD20ppsn,AntipovD20ppsn,YeWDB20}. The heavy-tailed mutation operator is nothing else than standard bit mutation with a random mutation rate, chosen from a heavy-tailed distribution. In~\cite{DoerrLMN17}, it was defined as follows. Let $\beta > 1$ be a constant. This will be the only parameter of the mutation operator, however, one with not too much importance, so~\cite{DoerrLMN17} simply propose to take $\beta = 1.5$. In each invocation of the mutation operator, a number $\alpha \in [1..N]$, $N := \lfloor \frac 12 n \rfloor$, is chosen from the power-law distribution with exponent $\beta$. Hence $\Pr[\alpha = i] = (C^\beta_{N})^{-1} i^{-\beta}$, where $C^\beta_{N}$ is the normalizing constant $C^\beta_N := \sum_{i=1}^N i^{-\beta}$. Once $\alpha$ is determined, standard bit mutation with mutation rate $p = \frac \alpha n$ is employed.

For fixed $N$, the expression on the left-hand side in Lemma~\ref{lem:randsbm} is $A_N = \sum_{i=1}^N (C^\beta_{N})^{-1} i^{-\beta} e^{-i}$. This is a convex combination of $e^{-i}$ terms and by comparing the coefficients, we easily see that this expression is decreasing in $N$. Computing $A_{100} < 0.178$, we see that for all $n \ge 200$, a reproduction number of at most $\alpha \le \frac{1}{0.178} \approx 5.618$ is small enough to lead to exponential runtimes. This is higher than for standard bit mutation with mutation rate $p = \frac 1n$, where only $\alpha \le e \approx 2.718$ suffices to show exponential runtimes. This observation fits to our general feeling that larger mutation rates can be destructive, from which in particular non-elitist algorithms suffer. 

For the limiting value $A = \lim_{N \to \infty} A_N$ we note that $A \ge A^{-} := \sum_{i=1}^{100} (C^\beta_\infty)^{-1} i^{-\beta} \exp(-i) \approx 0.164004$ and $A \le A^+ := \sum_{i=1}^{100} (C^\beta_\infty)^{-1} i^{-\beta} \exp(-i) + \exp(-101) \approx 0.164004$. Hence for $n$ sufficiently large, even a value of $\alpha \approx \frac{1}{0.164004} = 6.0974$ admits exponential lower bounds for runtimes.

Without going into details, and in particular without full proofs, we note that these estimates are tight, and this for all mutation operators of the type discussed in this section. Consider such a mutation operator such that $\sum_{i=1}^m q_i \exp(-p_i n) \ge (1+\delta) \frac 1 \alpha$ for some $\delta > 0$. We take the \mclea optimizing \onemax as example. Assume that at some time we have in our parent population $k$ individuals on the highest non-empty fitness level $L$. In expectation, each of them generates $\alpha = \lambda / \mu$ offspring. Each of these offspring is an exact copy of the parent with probability $\sum_{i=1}^m q_i \exp(-p_i n) \ge (1+\delta) \frac 1 \alpha$. Consequently, in the next generation the expected number of individuals on level $L$ or higher (as long as level $L$ is not full) is $(1+\delta) k$. This is enough to show polynomial runtimes via the level-based method~\cite{Lehre11,DangL16algo,CorusDEL18,DoerrK19} when $\lambda$ is large enough. For example, the computation just made shows that condition~(G2) in \cite[Theorem~3.2]{DoerrK19} is satisfied.

We note that this tightness stems from the fact that the term $\sum_{i=1}^m q_i \exp(-p_i n)$ appears both in the drift computation here and, via the probability to generate a copy of a parent, in the fitness level method for populations. We do not think that this is a coincidence, but leave working out the details to a future work.

\section{Fitness Proportionate Selection}\label{sec:fp}

In this section, we apply our method to a mutation-only version of the simple genetic algorithm (simple GA). We note that this algorithm traditionally is used with crossover~\cite{Goldberg89}. The mutation-only version has been regarded in the runtime analysis community mostly because runtime analyses for crossover-based algorithms are extremely difficult. While the first runtime analysis for the mutation-only version~\cite{NeumannOW09} appeared in 2009 and showed a near-exponential lower bound on \onemax for arbitrary polynomially bounded population sizes, the first analysis of the crossover-based version from 2012~\cite{OlivetoW12gecco} could only show a significantly sub-exponential lower bound ($2^{n^c}$ for a constant $c$ which is at most $\frac 1 {80}$) and this for population sizes below $n^{1/8}$. We note that the current best result~\cite{OlivetoW15} gives a similar lower bound for population sizes below $n^{1/4}$. Both works call these runtimes exponential, and we acknowledge that this definition for exponential runtimes exists, but given the substantial difference between $2^{n^{1/80}}$ (which is less than $3.5$ for all $n \le 10^{20}$) and $\exp(\Theta(n))$ we prefer to reserve the notion ``exponential'' for the latter. 

The mutation-only version of the simple GA with population size $\mu \in \N$ is described in Algorithm~\ref{alg:simpleGA}. This algorithm starts with a population $P^{(0)}$ of $\mu$ random individuals from $\{0,1\}^n$. In each iteration $t = 1, 2, 3, \dots$, it computes from the previous population $P^{(t-1)}$ a new population $P^{(t)}$ by $\mu$ times independently selecting an individual from $P^{(t-1)}$ via fitness proportionate selection and mutating it via standard bit mutation with mutation rate $p = \frac 1n$. 

\begin{algorithm2e}%
	Initialize $P^{(0)}$ with $\mu$ individuals chosen independently and uniformly at random from $\{0,1\}^n$\;
	\For{$t = 1, 2, \ldots$}{
    \For{$i \in [1..\mu]$}{
      Select $x \in P^{(t-1)}$ via fitness proportionate selection\;
      Generate $P^{(t)}_i$ from $x$ via standard bit mutation\;
      }
  }
\caption{The simple genetic algorithm (simple GA) with population size $\mu$ to maximize a function $f : \{0,1\}^n \to \R_{\ge 0}$.}
\label{alg:simpleGA}
\end{algorithm2e}

The precise known results for the performance of Algorithm~\ref{alg:simpleGA} on the \onemax benchmark are the following. \cite[Theorem~8]{NeumannOW09} showed that with $\mu \le \poly(n)$ it needs with high probability more than $2^{n^{1-O(1/\log\log n)}}$ iterations to find the optimum of the \onemax function or any search point in Hamming distance at most $0.003n$ from it. This is only a sub\-exponential lower bound. In~\cite[Corollary~13]{Lehre11}, building on the lower bound method from~\cite{Lehre10}, a truly exponential lower bound is shown for the weaker task of finding a search point in Hamming distance at most $0.029n$ from the optimum, but only for a relatively large population size of $\mu \ge n^3$ (and again $\mu \le \poly(n)$). 

We now extend this result to arbitrary $\mu$, that is, we remove the conditions $\mu \ge n^3$ and $\mu \le \poly(n)$. To obtain the best known constant $0.029$ for how close to the optimum the algorithm cannot go in sub\-exponential time, we have to compromise with the constants in the runtime, which consequently are only of a theoretical interest. We therefore do not specify the base of the exponential function or the leading constant. We note that this would have been easily possible since we only use a simple additive Chernoff bound and Corollary~\ref{cor:sbm}. We further note that Lehre~\cite{Lehre11} also shows lower bounds for a scaled version of fitness proportionate selection and a general $\Theta(1/n)$ mutation rate. This would also be possible with our approach and would again remove the conditions on $\lambda$, but we do not see that the additional effort is justified here. 

\begin{theorem}\label{thm:fp}
  There is a $T = \exp(\Omega(n))$ such that the mutation-only simple GA optimizing \onemax with any population size $\mu$ with probability $1 - \exp(-\Omega(n))$ does not find any solution $x$ with $\onemax(x) \ge 0.971n$ within $T$ fitness evaluations.
\end{theorem}

The main difficulty in proving lower bounds for algorithms using fitness proportionate selection is that the reproduction number is non-trivial to estimate. If all but one individual have a fitness of zero, then this individual is selected $\mu$ times. Hence $\mu$ is the only general upper bound for the reproduction number. The  previous works and ours overcome this difficulty by arguing that the average fitness in the population cannot significantly drop below the initial value of $n/2$, which immediately yields that an individual with fitness $k$ has a reproduction number of roughly at most $\frac{k}{n/2}$.

While it is natural that the typical fitness of an individual should not drop far below $n/2$, making this argument precise is not completely trivial. In~\cite[Lemma~6]{NeumannOW09}, it was informally argued that the situation with fitness proportionate selection cannot be worse than with uniform selection. For the latter situation a union bound over all lineages of individuals is employed and a negative-drift analysis from~\cite[Section~3]{OlivetoW08} is used for a single lineage. The analysis in~\cite[Lemma~9]{Lehre11} builds on the (positive) drift stemming from standard bit mutation when the fitness is below $n/2$ (this argument needs a mutation rate of at least $\Omega(1/n)$) and the independence of the offspring (here the lower bound $\lambda \ge n^3$ is needed to admit the desired Chernoff bound estimates). 

Our proof relies on a natural domination argument which shows that at all times all individuals are at least as good as random individuals in the sense of stochastic domination in fitness. This allows to use a simple Chernoff and union bound to argue that with high probability, for a long time all individuals have a fitness of at least $(\frac 12 - \eps) n$. The remainder of the proof is an application of Corollary~\ref{cor:sbm}. Here Lehre's lower bound~\cite[Theorem~4]{Lehre10} would have been applicable as well with the main difference that there one has to deal with the constant $\delta$, which does not exist in Corollary~\ref{cor:sbm}. 

We start by proving the key argument used in the proof of Theorem~\ref{thm:fp}, namely that at each time $t$ for each individual $i \in [1..\mu]$ the fitness stochastically dominates (see Section~\ref{sec:prelims}) the one of a random individual. We denote by $\Bin(n,p)$ the binomial distribution with parameters $n$ and $p$. With a slight abuse of notation, we write ${\Bin(n,p) \preceq Y}$ to denote that $Y$ stochastically dominates $X$ when $X$ is binomially distributed with parameters $n$ and $p$. 

In this notation, our goal is to show that for all times $t$ and all $i \in [1..\mu]$, we have $\Bin(n,\frac 12) \preceq \onemax(P^{(t)}_i)$. This statement appears easy to believe since fitness proportionate selection, favoring better individuals at least slightly, should not be able to make the population worse. To be on the safe side, we nevertheless prove this statement formally (after the following remark). 

We note that another statement that might be easy to believe is not true, namely that the sum of the fitness values of a population at all times $t \ge 1$ dominates the sum of the fitness values of a random population (such as the initial population), that is, that $\Bin(\mu n, \frac 12) \preceq \sum_{i=1}^\mu \onemax(P^{(t)}_i)$. As counter-example, let $n$ be a multiple of $10$ and let us consider the simple GA with $\mu = n$ after one iteration. Let $Y = \sum_{i=1}^\mu \onemax(P^{(1)}_i)$. We estimate the probability of the event $Y \le 0.4 n \mu$. With probability at least $2^{-n}$ we have $\onemax(P^{(0)}_1) = 0.4n$. For each $i = 2, \dots, \mu$, we have $\onemax(P^{(0)}_i) \le 0.5n$ with probability $0.5$ by the symmetry of the binomial distribution with parameter $p=0.5$. These events are all independent, so with probability at least $2^{-n - \mu + 1}$, we have all of them. In this case, for each $i = 1, \dots, \mu$ independently, with probability at least $0.4n / (0.4n + 0.5n(\mu-1)) \ge 0.8 / \mu$ the $i$-th parent chosen in iteration $1$ is $P^{(0)}_1$ and with probability at least $(1 - 1/n)^n \ge 1/4$ the offspring generated from it equals the parent. All these events together occur with probability at least $2^{-n - \mu + 1} (0.8/\mu)^\mu (1/4)^\mu \ge (20n)^{-n}$, recall that $\mu = n$, which shows $\Pr[Y \le 0.04 n^2] \ge (20n)^{-n}$. Now for $X \sim \Bin(n\mu,\frac 12)$, a simple Chernoff bound argument, e.g., via the additive Chernoff bound~\cite[Theorem~1.10.7]{Doerr20bookchapter}, shows that $\Pr[X \le 0.4n\mu] \le \exp(2(0.1n\mu)^2 / n\mu) = \exp(-0.02 n\mu) = \exp(-0.02 n^2)$. Since this is (much) smaller than $(20n)^{-n}$ for $n$ sufficiently large ($n \ge 180$ suffices), we do not have $X \preceq Y$.

\begin{lemma}\label{lem:fpdom}
  Consider a run of the simple GA (Algorithm~\ref{alg:simpleGA}) on the \onemax benchmark. Then for each $t \ge 0$ and each $i \in [1..\mu]$, we have $\Bin(n,\frac 12) \preceq \onemax(P^{(t)}_i)$.
\end{lemma}

To prove this result, we use the following auxiliary result, which states a number uniformly chosen from a collection of non-negative numbers is stochastically dominated by a number chosen from the same collection via an analogue of fitness proportionate selection. To define the latter formally, let $n_1, \dots, n_\mu \in \R_{\ge 0}$. For a random variable $v$ we write $v \sim \fp(n_1, \dots, n_\mu)$ if
\begin{itemize}
\item in the case that $u_i > 0$ for at least one $i \in [1..\mu]$, we have $\Pr[v = i] = \frac{n_i}{\sum_{j=1}^\mu n_j}$ for all $i \in [1..\mu]$, and
\item in the case that $u_i = 0$ for all $i \in [1..\mu]$, we have $\Pr[v = i] = \frac 1\mu$ for all $i \in [1..\mu]$. 
\end{itemize}

\begin{lemma}\label{lem:unifp}
  Let $n_1, \dots, n_\mu \in \R_{\ge 0}$. Let $u \in [1..\mu]$ be uniformly chosen and $U = n_u$. Let $v \sim \fp(n_1, \dots, n_\mu)$ and $V = n_v$. Then $U \preceq V$.
\end{lemma}

\begin{proof}
  The claim follows immediately from the definition of $\fp(\cdot)$ when $n_i = 0$ for all $i \in [1..\mu]$. Hence let us assume that there is at least one $i \in [1..\mu]$ such that $n_i > 0$. Let us for convenience assume that $n_1 \le n_2 \le \dots \le n_\mu$. Then apparently 
  \[\frac{1}{i} \sum_{j=1}^i n_j \le \frac{1}{\mu} \sum_{j=1}^\mu n_j\]
and hence 
  \[\Pr[V \le n_i] = \frac{\sum_{j=1}^i n_j}{\sum_{j=1}^\mu n_j} \le \frac i \mu = \Pr[U \le n_i]\]
for all $i \in [1..\mu]$. This suffices to show stochastic domination since both $U$ and $V$ only take the values $n_1, \dots, n_\mu$. 
\end{proof}

We now show Lemma~\ref{lem:fpdom}.

\begin{proof}[Proof of Lemma~\ref{lem:fpdom}]
  We show the claim via induction over time. For the random initial population $P^{(0)}$, the claim is obviously true. Assume that in some iteration $t+1$, the parent population $P^{(t)}$ satisfies that for all $i \in [1..\mu]$, we have $\Bin(n,\frac 12) \preceq \onemax(P^{(t)}_i)$. We show that the same is true for $P^{(t+1)}$. Since all individuals of $P^{(t+1)}$ are identically distributed, we consider how one of them is generated. Let $u \in [1..\mu]$ be random and let $v \sim \fp(\onemax(P^{(t)}_1), \dots, \onemax(P^{(t)}_\mu))$ be the parent individual selected for the generation of the offspring. By our inductive assumption and Lemma~\ref{lem:unifp}, we have $\Bin(n,\frac 12) \preceq \onemax(P^{(t)}_u) \preceq \onemax(P^{(t)}_v)$. Let $x$ be a uniformly random individual and $y = P^{(t)}_v$ be the parent just selected. Let $x'$ and $y'$ be the results of applying standard bit mutation to $x$ and $y$. Since $\onemax(x) \preceq \onemax(y)$, by Lemma~\ref{lem:witt} we have $\onemax(x') \preceq \onemax(y')$. Now $y'$ is equal (in distribution) to the offspring we just regard and $x'$ is (still) a random bit string. Hence $\Bin(n,\frac 12) \sim \onemax(x') \preceq \onemax(y')$ as desired. 
\end{proof}

We are now  ready to give the formal proof of Theorem~\ref{thm:fp}.

\begin{proof}[Proof of Theorem~\ref{thm:fp}]
  Consider a run of the simple GA with population size~$\mu$. With the domination argument of Lemma~\ref{lem:fpdom}, the fitness of a particular solution $P^{(t)}_i$ dominates a sum of $n$ independent uniform $\{0,1\}$-valued random variables. Hence using the additive Chernoff bound (see, e.g., \cite[Theorem~1.10.7]{Doerr20bookchapter}), we see that $\onemax(P^{(t)}_i) \le (\frac 12 - \eps) n =: s$ with probability at most $\exp(-2 \eps^2 n)$ for all $\eps>0$.
  
  To avoid working in conditional probability spaces, let us consider a modification of the simple GA. It is identical with the original algorithm up to the point when the fitness of an individual $P^{(t)}_i$ for the first time goes below~$s$. From that time on, the algorithm selects the parents uniformly. Such an artificial continuation of a process from a time beyond the horizon of interest on was, to the best of our knowledge, in the theory of evolutionary algorithms first used in~\cite{DoerrHK11}. For our modified simple GA, the reproduction rate of any individual in a population with all elements having fitness less than $n-a$, $a \le n-s$,  is at most $\frac{n-a}{s} =: \alpha$. Hence we can apply Corollary~\ref{cor:sbm} with this $\alpha$. Taking, similar as in~\cite{Lehre11}, $\eps = 0.0001$ and $a = 0.029$, we can work with $\alpha = \frac{1-a}{0.5 - \eps} \approx 1.942388$ 
and thus $\gamma \approx 0.336082$. 
For $n$ sufficiently large, this allows to use $b = \lceil 0.02905 n\rceil$.  
  
  For the first time $T'$ that the modified algorithm finds a solution with fitness at least $n-a$ we thus obtain 
  \[\Pr[T' < L] = \frac{2L\mu n}{p \alpha} \max\left\{1, \frac{p\alpha}{2n}\right\} \exp\left(-\ln\left(\frac 2\gamma \right)(b-a)\right) = L \mu \exp(-\Omega(n))\] 
  for all~$L$. Since with probability at least $1 - L \mu \exp(-2 \eps^2 n)$ the modified and the true algorithm do not differ in the first $L$ iterations (union bound over all individuals generated in this time interval), we have $\Pr[T < L] \le L \mu \exp(-\Omega(n)) + L \mu \exp(-\Omega(n)) = L \mu \exp(-\Omega(n))$. With $L = \exp(\Theta(n))/\mu$ suitably chosen, we have shown the claim (note that each iteration takes $\mu$ fitness evaluations and note further that we can assume $\mu = \exp(O(n))$ sufficiently small as otherwise the evaluation of the initial search points already proves the claim).
\end{proof}

While an exponential runtime on \onemax is not an exciting performance, for the sake of completeness we note that the runtime of the simple GA on \onemax is not worse than exponential. A runtime of $\exp(O(n))$ can be shown with the methods of~\cite{Doerr20ppsnUB} (with some adaptations). The key observation is that, similar to property~(A) in~\cite[Theorem~3]{Doerr20ppsnUB}, if at some time $t$ the population contains an individual $x$ with some fitness at least $n/3$, then in the next iteration this individual is chosen as parent at least once with at least constant probability and, conditional on this, with probability $\Omega(\frac 1n)$ a particular better Hamming neighbor of $x$ is generated from $x$. 

\section{Conclusion and Outlook}\label{sec:conclusions}

In this work, we have proven two technical tools which might ease future lower bound proofs in discrete evolutionary optimization. The negative multiplicative drift theorem has the potential to replace the more technical negative drift theorems used so far in different contexts. Our strengthening and simplification of the negative drift in populations method should help increasing our not very well developed understanding of population-based algorithms in the future. Clearly, it is restricted to mutation-based algorithms -- providing such a tool for crossover-based algorithms and extending our understanding how to prove lower bounds for these beyond the few results~\cite{DoerrT09,OlivetoW15,SuttonW19,Doerr20symmetryarxiv} would be a great progress. 

\subsection*{Acknowledgement}

This work was supported by a public grant as part of the
Investissement d'avenir project, reference ANR-11-LABX-0056-LMH,
LabEx LMH.


\newcommand{\etalchar}[1]{$^{#1}$}

}
\end{document}